\documentclass[11pt]{article}
\usepackage{amsmath,amsfonts,amssymb}
\usepackage{authblk}
\usepackage{algorithmic}
\usepackage{algorithm}
\usepackage{array,soul}
\usepackage{booktabs}
\usepackage{textcomp}
\usepackage{stfloats,colortbl,xcolor,graphicx}
\usepackage{url,hyperref}
\usepackage{verbatim}
\usepackage{graphicx}
\usepackage{rsfso}
\usepackage{enumerate}
\usepackage{array}

\usepackage{lastpage}
\newcommand{\BlackBox}{\rule{1.5ex}{1.5ex}}  
\ifdefined\proof
    \renewenvironment{proof}{\par\noindent{\bf Proof\ }}{\hfill\BlackBox\\[2mm]}
\else
    \newenvironment{proof}{\par\noindent{\bf Proof\ }}{\hfill\BlackBox\\[2mm]}
\fi
 
\newtheorem{theorem}{Theorem}
\newtheorem{lemma}[theorem]{Lemma}

\begin{document}

\title{Fast Debiasing of the LASSO Estimator}

\author[1]{Shuvayan Banerjee\thanks{banerjeeshuvayan21@gmail.com}}
\author[2]{James Saunderson\thanks{James.Saunderson@monash.edu}}
\author[3]{Radhendushka Srivastava\thanks{rsrivastava@iitb.ac.in}}
\author[4]{Ajit Rajwade\thanks{ajitvr@cse.iitb.ac.in}}
\affil[1]{Department of Mathematics, IIT Bombay, and IIT Bombay Monash Research Academy}
\affil[2]{Department of Electrical and Computer Systems Engineering, Monash University}
\affil[3]{Department of Mathematics, IIT Bombay}
\affil[4]{Department of Computer Science and Engineering, IIT Bombay}

\maketitle

\begin{abstract}
    In high-dimensional sparse regression, the \textsc{Lasso} estimator offers excellent theoretical guarantees but is well-known to produce biased estimates. To address this, \cite{Javanmard2014} introduced a method to ``debias" the \textsc{Lasso} estimates for a random sub-Gaussian sensing matrix $\boldsymbol{A}$. Their approach relies on computing an ``approximate inverse" $\boldsymbol{M}$ of the matrix $\boldsymbol{A}^\top \boldsymbol{A}/n$ by solving a convex optimization problem. This matrix $\boldsymbol{M}$ plays a critical role in mitigating bias and allowing for construction of confidence intervals using the debiased \textsc{Lasso} estimates. However the computation of $\boldsymbol{M}$ is expensive in practice as it requires iterative optimization.
In the presented work, we re-parameterize the optimization problem to compute a ``debiasing matrix" $\boldsymbol{W} := \boldsymbol{AM}^{\top}$ directly, rather than the approximate inverse $\boldsymbol{M}$. This reformulation retains the theoretical guarantees of the debiased \textsc{Lasso} estimates, as they depend on the \emph{product} $\boldsymbol{AM}^{\top}$ rather than on $\boldsymbol{M}$ alone. Notably, we provide a simple, computationally efficient, closed-form solution for $\boldsymbol{W}$ under similar conditions for the sensing matrix $\boldsymbol{A}$ used in the original debiasing formulation, with an additional condition that the elements of every row of $\boldsymbol{A}$ have uncorrelated entries.
Also, the optimization problem based on $\boldsymbol{W}$ guarantees a unique optimal solution, unlike the original formulation based on $\boldsymbol{M}$. We verify our main result with numerical simulations.
\end{abstract}

\section{Introduction}
In high-dimensional sparse regression, where the number of predictors significantly exceeds the number of observations, the \textsc{Lasso} (Least Absolute Shrinkage and Selection Operator) is a widely used method for variable selection and estimation. By incorporating an $\ell_1$ regularization term, \textsc{Lasso} promotes sparsity in the estimated coefficients, enabling effective performance for sparse signal vectors even if the number of predictors far exceeds the number of samples. The \textsc{Lasso} estimator has well-established theoretical guarantees for signal and support recovery \cite{THW2015}. Despite its strengths, a well-recognized limitation of \textsc{Lasso} is its tendency to produce biased estimates. This bias arises from the shrinkage imposed by the $\ell_1$ penalty. Consequently, the bias compromises estimation accuracy and impedes statistical inference tasks such as construction of confidence intervals or hypothesis tests. These challenges are especially pronounced in high-dimensional regimes, where traditional inference tools fail due to high dimensionality.

To address these limitations, several methods have been developed to ``debias" the \textsc{Lasso} estimator, allowing for valid statistical inference even in high-dimensional settings. Notably, \cite{zhangzhang} introduced a decorrelated score-based approach, leveraging the Karush–Kuhn–Tucker (KKT) conditions of the \textsc{Lasso} optimization problem to construct bias-corrected estimators. Their framework relies on precise estimation of the precision matrix (inverse covariance matrix), which can be computationally challenging and sensitive to regularization choices. Similarly, \cite{vandegeer2014} proposed a methodology rooted in node-wise regression, where each variable is regressed on the remaining variables to estimate the precision matrix. While effective, this method is computationally intensive.
This may limit its applicability, particularly in scenarios where the design matrix lacks favorable properties like sparsity of the rows of the precision matrix.

\cite{Javanmard2014} introduced a simple yet powerful approach that constructs debiased \textsc{Lasso} estimates using an ``approximate inverse" of the sample covariance matrix. Their method avoids direct precision matrix estimation and instead employs an optimization framework to compute a debiasing matrix $\boldsymbol{M}$ that corrects for bias while ensuring asymptotic normality of the debiased estimates. A key advantage of this method is its 
applicability for random sub-Gaussian sensing matrices, enabling valid inference across a broad range of high-dimensional applications.

In this work, we build upon the technique of \cite{Javanmard2014}, addressing one of its primary computational bottlenecks: the optimization step required to compute the approximate inverse $\boldsymbol{M}$. By reformulating the problem to work directly with the ``weight matrix" $\boldsymbol{W} := \boldsymbol{AM^\top}$, we entirely eliminate the need to solve this optimization problem in many practical cases. Our proposed reformulation leverages the insight that the theoretical guarantees of the debiased \textsc{Lasso} estimator depend on the product $\boldsymbol{AM}^\top$ rather than the individual debiasing matrix $\boldsymbol{M}$. By shifting the focus to the ``weight matrix" $\boldsymbol{W} := \boldsymbol{AM}^\top$, we simplify the optimization problem while retaining all theoretical properties of the original framework. Under certain deterministic assumptions, we provide a simple, exact, closed form optimal solution for the optimization problem to obtain $\boldsymbol{W}$. We show that this assumption is satisfied with high probability for the ensembles of sensing matrices considered in \cite{Javanmard2014}, under the additional condition that the elements of the rows of $\boldsymbol{A}$ are uncorrelated. In practice, sensing matrices with uncorrelated entries are commonly used in many applications \cite{duarte2008single,liu2013efficient} and are also widely used in many theoretical results in sparse regression \cite{THW2015}. This closed form solution eliminates the computationally intensive optimization step required to compute $\boldsymbol{M}$, significantly improving runtime efficiency. It is  applicable in many natural situations, including sensing matrices with i.i.d.\ isotropic sub-Gaussian rows (such as i.i.d.\ Gaussian, or i.i.d.\ Rademacher entries).

\paragraph{Notation:} Throughout this paper, we denote matrices by bold-faced uppercase symbols, e.g., $\boldsymbol{A}$. If $\boldsymbol{A}$ is an $n\times p$ matrix then $\boldsymbol{a}_{i.}\in \mathbb{R}^p$ denotes the $i^{\textrm{th}}$ row of $\boldsymbol{A}$, thought of as a column vector. Similarly if $\boldsymbol{A}$ is an $n\times p$ matrix then $\boldsymbol{a}_{.j}\in \mathbb{R}^n$ denotes the $j^{\textrm{th}}$ column of $\boldsymbol{A}$, again thought of as a column vector. Vectors are denoted by bold-faced lower case symbols, e.g., $\boldsymbol{w}$. The $i$th entry of a vector $\boldsymbol{w}$ is denoted $w_i\in \mathbb{R}$. The identity matrix of size $p \times p$ for any positive integer $p$ is denoted by $\boldsymbol{I_p}$, and its $i^{\textrm{th}}$ column vector is denoted by $\boldsymbol{e}_i$. 
For a positive integer $p$, we use the shorthand $[p] = \{1,2,\ldots,p\}$. For a vector $\boldsymbol{w}\in \mathbb{R}^m$, we denote the $\ell_q$-norm by $\|\boldsymbol{w}\|_q := \left(\sum_{i=1}^{m}|w_i|^q\right)^{1/q}$ if $1\leq q < \infty$ and the $\ell_{\infty}$-norm by $\|\boldsymbol{w}\|_{\infty}:= \max_{i\in [m]} |w_i|$.  

\section{An Overview of the Debiased LASSO}
We consider the high-dimensional linear model
\begin{equation}\label{eq:forward_model}
    \boldsymbol{y} = \boldsymbol{A\beta^*} + \boldsymbol{\eta},
\end{equation} 
where $\boldsymbol{\beta^*} \in \mathbb{R}^p$ is a $s$-sparse signal (i.e., $s := \|\boldsymbol{\beta^*}\|_0$ where $s \ll p$), $\boldsymbol{A}$ is a $n \times p$ design/sensing matrix (where $n \ll p$), and $\boldsymbol{y} \in \mathbb{R}^n$ is the measurement vector. Also, $\boldsymbol{\eta} \in \mathbb{R}^n$ is an additive noise vector that consists of independent and identically distributed elements drawn from $\mathcal{N}(0, \sigma^2)$, where $\sigma^2$ is the noise variance.

The \textsc{Lasso} estimate $\boldsymbol{\hat{\beta}_{\lambda}}$ of the sparse signal $\boldsymbol{\beta^*}$ is defined as the solution to the following optimization problem:  
\begin{equation}\label{eq:lasso}
    \boldsymbol{\hat{\beta}_{\lambda}} := \arg \min_{\boldsymbol{\beta}} \frac{1}{2n}\|\boldsymbol{y} - \boldsymbol{A\beta}\|_2^2 + \lambda \|\boldsymbol{\beta}\|_1,
\end{equation}  
where $\lambda > 0$ is a regularization parameter chosen appropriately.  
The \textsc{Lasso} estimator is known to be a consistent estimator of the sparse signal $\boldsymbol{\beta^*}$ under the condition that the sensing matrix $\boldsymbol{A}$ satisfies the Restricted Eigenvalue Condition (REC) \cite[Chapter 11]{THW2015}.

\subsection{Debiasing the LASSO Estimator}
\label{subsec:debias_LASSO}
The \textsc{Lasso} estimator is well-known to produce biased estimates, i.e., $E(\boldsymbol{\hat{\beta}_{\lambda}}) \neq \boldsymbol{\beta^*}$ where the expectation is computed over noise instances. This bias arises from the $\ell_1$ regularization term, which induces shrinkage in the estimate $\boldsymbol{\hat{\beta}_{\lambda}}$. Moreover, there is no known method to compute a confidence interval of $\boldsymbol{\beta^*}$ directly from $\boldsymbol{\hat{\beta}_{\lambda}}$. 

To reduce this bias and also construct confidence intervals of $\boldsymbol{\beta^*}$, \cite{Javanmard2014} introduced a debiased \textsc{Lasso} estimator $\boldsymbol{\hat{\beta}_{d}}$, defined as follows:  
\begin{equation}
\boldsymbol{\hat{\beta}_{d}} := \boldsymbol{\hat{\beta}_{\lambda}} + \frac{1}{n} \boldsymbol{M} \boldsymbol{A}^\top (\boldsymbol{y} - \boldsymbol{A} \boldsymbol{\hat{\beta}_{\lambda}}).
\label{eq:debiased_beta1}
\end{equation}  
Here $\boldsymbol{M}$ is an approximate inverse of the rank deficient matrix $\boldsymbol{\hat{\Sigma}} := \boldsymbol{A}^\top \boldsymbol{A}/n$, computed by solving the convex optimization problem given in Algorithm~\ref{alg:M}. The theoretical properties of $\boldsymbol{\hat{\beta}_d}$ from \cite{Javanmard2014} are applicable to a sensing matrix $\boldsymbol{A}$ with the following properties:
\begin{enumerate}
\item [\textbf{D1:}]  The rows $\boldsymbol{a}_{1.},\boldsymbol{a}_{2.},\ldots,\boldsymbol{a}_{n,.}$ of matrix  $\boldsymbol{A}$ are independent and identically distributed zero-mean sub-Gaussian random vectors with covariance $\boldsymbol{\Sigma} :=E[\boldsymbol{a}_{i.}\boldsymbol{a}_{i.}^\top]$. 
Furthermore, the sub-Gaussian norm $\kappa :=\|\boldsymbol{\Sigma}^{-1/2} \boldsymbol{a}_{i.}\|_{\psi_2}$ \footnote{The sub-Gaussian norm of a random variable $x$, denoted by $\|x\|_{\psi_2}$, is defined as $\|x\|_{\psi_2} := \sup_{q \geq 1} q^{-1/2} \left({E}|x|^q \right)^{1/q}$. For a random vector $\boldsymbol{x} \in \mathbb{R}^n$, its sub-Gaussian norm is defined as $\|\boldsymbol{x}\|_{\psi_2} := \sup_{\boldsymbol{y} \in S^{n-1}} \| \boldsymbol{y^\top x}\|_{\psi_2}$, where $S^{n-1}$ denotes the unit sphere in $\mathbb{R}^n$.} is a finite positive constant. 
\item [\textbf{D2:}] There exist positive constants $0<C_{\min}\leq C_{\max}$, such that the minimum and maximum eigenvalues $\sigma_{\min}(\boldsymbol{\Sigma}), \sigma_{\max}(\boldsymbol{\Sigma})$ of $\boldsymbol{\Sigma}$ satisfy $0 < C_{\min} \leq \sigma_{\min}(\boldsymbol{\Sigma}) \leq \sigma_{\max}(\boldsymbol{\Sigma}) \leq C_{\max} < \infty $.
\end{enumerate}

Theorem 7(b) of \cite{Javanmard2014} shows that the optimization problem in~\eqref{eq:opt_problem} to obtain $\boldsymbol{M}$ is feasible with high probability, for sensing matrices satisfying properties \textbf{D1} and \textbf{D2}, as long as $\mu > 4\sqrt{3}e\kappa^2 \sqrt{\frac{C_{\max}}{C_{\min}}} \sqrt{\frac{\log p}{n}}$. If $\mu$  is $O\left(\sqrt{\frac{\log p}{n}}\right)$ and $n$ is $\omega((s \log p)^2)$, then Theorem 8 of \cite{Javanmard2014} shows that $\forall j \in [p], \sqrt{n}({\hat{\beta}}_{dj}-{\beta^*_j})$ is asymptotically zero-mean Gaussian with variance $\sigma^2 \boldsymbol{m}_{.j}^\top \boldsymbol{\hat{\Sigma}} \boldsymbol{m}_{.j}$.

\begin{algorithm}[H] 
\caption{Construction of $\boldsymbol{M}$ (from \cite{Javanmard2014})}
\label{alg:M}
\begin{algorithmic}[1]
\REQUIRE 
Design matrix $\boldsymbol{A}$, $\mu \in (0,1)$  
\ENSURE 
Debiasing matrix $\boldsymbol{M}$  
\STATE Compute: $\boldsymbol{\hat{\Sigma}} := \boldsymbol{A}^\top \boldsymbol{A}/n$.  
\STATE For each $j \in [p]$, solve the following optimization problem to compute column vector $\boldsymbol{m}_{.j} \in \mathbb{R}^p$:  
\begin{eqnarray}
\label{eq:opt_problem}
\nonumber \text{minimize} \quad  \boldsymbol{m}_{.j}^\top \boldsymbol{\hat{\Sigma}} \boldsymbol{m}_{.j} \\
\text{subject to} \quad \|\boldsymbol{\hat{\Sigma}} \boldsymbol{m}_{.j} - \boldsymbol{e}_j\|_{\infty} \leq \mu,
\end{eqnarray}
where $\boldsymbol{e}_j$ is the $j^\text{th}$ column of the identity matrix $\boldsymbol{I_p}$, and $\mu \in (0,1)$.  
\STATE Assemble $\boldsymbol{M}$ as $\boldsymbol{M} := (\boldsymbol{m}_{.1}|\cdots|\boldsymbol{m}_{.p})^\top$.  
\STATE If the optimization problem is infeasible for any $j$, set $\boldsymbol{M} := \boldsymbol{I_p}$.  
\end{algorithmic}
\end{algorithm}  
\section{Re-parameterization of the Debiased LASSO}
The debiased \textsc{Lasso} estimator in~\eqref{eq:debiased_beta1} can be rewritten in terms of the variable $\boldsymbol{W}:= \boldsymbol{A}\boldsymbol{M}^\top$  as:
\begin{equation}
\boldsymbol{\hat{\beta}_{d}} = \boldsymbol{\hat{\beta}_{\lambda}} + \frac{1}{n} \boldsymbol{W}^\top (\boldsymbol{y} - \boldsymbol{A \hat{\beta}_{\lambda}}).
\label{eq:debiased_beta_W}
\end{equation}  
The re-parameterization does not affect the debiasing procedure introduced in \cite{Javanmard2014}. Thus, any theoretical guarantees established using $\boldsymbol{M}$ extend to those using $\boldsymbol{W}$. 

We now produce a reformulated problem in \eqref{eq:opt_W_prim} using $\boldsymbol{W}$, and show that it is equivalent to the original optimization problem in Algorithm~\ref{alg:M}. 
Using the relationship $\boldsymbol{W} = \boldsymbol{AM^\top}$, we can rewrite $\boldsymbol{m}_{.j}$ as $\boldsymbol{w}_{.j} := \boldsymbol{Am}_{.j}$. Making this substitution, the objective in  \eqref{eq:opt_problem} becomes  
$\boldsymbol{m}_{.j}^\top \boldsymbol{\hat{\Sigma}} \boldsymbol{m}_{.j} = \frac{1}{n} \boldsymbol{w}_{.j}^\top \boldsymbol{w}_{.j}$ and the constraint  
 $\|\boldsymbol{\hat{\Sigma} m}_{.j} - \boldsymbol{e}_j\|_{\infty} \leq \mu$ (where $\boldsymbol{e}_j$ is the $j$th column of the identity matrix) becomes   
$\left\| \frac{1}{n} \boldsymbol{A}^\top \boldsymbol{w}_{.j} - \boldsymbol{e}_j \right\|_{\infty} \leq \mu.$
{This change of variables suggests the following reformulated optimization problem~\eqref{eq:opt_W_prim} for the $j^{\text{th}}$ column of $\boldsymbol{W}$:
\begin{eqnarray}\label{eq:opt_W_prim}
    \nonumber \mathcal{P}_j:= \textrm{minimize} & \quad  \frac{1}{n} \boldsymbol{w}_{.j}^{\top}\boldsymbol{w}_{.j} \\
    \textrm{subject to} & \quad \left\| \frac{1}{n} \boldsymbol{A}^\top \boldsymbol{w}_{.j} - \boldsymbol{e_j} \right\|_{\infty} \leq \mu.
\end{eqnarray}
In fact, the $j^{\text{th}}$ reformulated problem~\eqref{eq:opt_W_prim} and the $j^\text{th}$ original problem~\eqref{eq:opt_problem} are equivalent in the following sense: If $\boldsymbol{m}_{.j}$ is feasible for~\eqref{eq:opt_problem} then $\boldsymbol{w}_{.j} := \boldsymbol{A}\boldsymbol{m}_{.j}$ is feasible for~\eqref{eq:opt_W_prim} and $\frac{1}{n}\boldsymbol{w}_{.j}^\top \boldsymbol{w}_{.j} = \boldsymbol{m}_{.j}^\top \boldsymbol{\hat{\Sigma}}\boldsymbol{m}_{.j}$.
Conversely, suppose that $\boldsymbol{w}_{.j}$ is feasible for~\eqref{eq:opt_W_prim}. If  $\boldsymbol{A}^{\dagger}$ is a pseudo-inverse of $\boldsymbol{A}$, then $\boldsymbol{m}_{.j} := \boldsymbol{A}^{\dagger} \boldsymbol{w}_{.j}$ is feasible for~\eqref{eq:opt_problem} since  $\boldsymbol{\hat{\Sigma}}\boldsymbol{m}_{.j} = \frac{1}{n}\boldsymbol{A}^\top\boldsymbol{A}\boldsymbol{m}_{.j} = \frac{1}{n}\boldsymbol{A}^\top\boldsymbol{w}_{.j}$.
Moreover, $\frac{1}{n}\boldsymbol{w}_{.j}^\top \boldsymbol{w}_{.j} = \boldsymbol{m}_{.j}^\top \boldsymbol{\hat{\Sigma}}\boldsymbol{m}_{.j}$, so both have the same objective values, establishing that~\eqref{eq:opt_problem} and~\eqref{eq:opt_W_prim} are equivalent.
This reformulation provides an equivalent separable problem for each column of $\boldsymbol{W}$, maintaining all theoretical guarantees while simplifying the representation of the debiasing procedure.

The reformulated problem \eqref{eq:opt_W_prim} has a \emph{unique} optimal solution because the objective function is strongly convex with convex constraints. In contrast, the original problem~\eqref{eq:opt_problem} does not have a unique solution. Indeed if $\boldsymbol{m}_{.j}$ is any solution to~\eqref{eq:opt_problem}, then we can add to it any element of the nullspace of $\boldsymbol{A}$ to obtain another solution to~\eqref{eq:opt_problem}. 

\subsection{A Closed-Form Solution for the Debiasing Matrix $\boldsymbol{W}$}
In this section, we demonstrate that, for a suitable choice of $\mu$, the optimal solution to the problem \eqref{eq:opt_W_prim} can be computed in closed form for a sensing matrix whose minimum column norm is strictly positive (which is true with probability 1 for random matrices).
To derive this result we write down the Fenchel dual of~\eqref{eq:opt_W_prim}, and appeal to weak duality. In particular, we explicitly find primal and dual feasible points with the same objective value, certifying that both are, in fact, optimal.

\begin{theorem}\label{thm:dual_javanmard_gen}
Let $\boldsymbol{A}$ be a $n\times p$ matrix with no column equal to zero. Define $\rho(\boldsymbol{A}) := \max_{i\neq j} \frac{|\boldsymbol{a}_{.i}^\top \boldsymbol{a}_{.j}|}{\|\boldsymbol{a}_{.j}\|_2^2}$. The optimal solution of~\eqref{eq:opt_W_prim} is given by
\begin{equation}
\boldsymbol{w}_{.j} := \frac{n(1-\mu)}{\|\boldsymbol{a_{.j}}\|_2^2}\boldsymbol{a}_{.j} \quad \text{for all $j\in [p]$}
\label{eq:exact}
\end{equation}
if and only if $\frac{\rho}{1+\rho} \leq \mu \leq 1$.
\end{theorem}
The proof of this theorem is given in Appendix~\ref{sec:proof_thm_dual}. For notational simplicity, we will denote $\rho(\boldsymbol{A})$ by $\rho$ in the rest of the paper.

\paragraph{Remarks:}
\begin{enumerate}
\item This theorem eliminates the requirement to execute an iterative optimization algorithm to obtain $\boldsymbol{W}$ (or an iterative optimization algorithm to obtain $\boldsymbol{M}$). This is because given $\boldsymbol{A}$, one can directly implement the optimal solution of Alg.~\ref{alg:M} in the form~\eqref{eq:exact} for all $j \in [p]$. This speeds up the implementation of the debiasing of \textsc{Lasso} for the ensemble of sensing matrices that satisfy the conditions of Theorem~\ref{thm:dual_javanmard_gen}. 

\item The condition $\rho/(1+\rho) \leq \mu \leq 1$ is necessary and sufficient for the closed-form expression in~\eqref{eq:exact} to be optimal~\eqref{eq:opt_W_prim}. Note, however, that it is possible that~\eqref{eq:opt_W_prim} is feasible for values of $\mu$ that are smaller than $\rho/(1+\rho)$. In these cases the optimal solution to~\eqref{eq:opt_W_prim} is not given by~\eqref{eq:exact}. This is empirically illustrated in Subsection~\ref{subsec:mu_line}.

\item The condition $\frac{\rho}{1+\rho}\leq \mu < 1$  can be satisfied whenever the column norms of $\boldsymbol{A}$ are strictly positive. Given the sensing matrix $\boldsymbol{A}$, the quantity $\frac{\rho}{1+\rho}$ can be computed exactly.

\item The solution in \eqref{eq:exact} is the optimal solution even when we choose $\mu=1$. The optimal solution in this case is the trivial solution $\boldsymbol{w}_j=0$. Hence, one should always choose $\mu<1$.



\end{enumerate}

Recall that as per Theorem 8 of \cite{Javanmard2014}, if $\mu$ is $O\left(\sqrt{\frac{\log p}{n}}\right)$ and $n$ is $\omega((s \log p)^2)$, then  $\forall j \in [p], \sqrt{n}({\hat{\beta}}_{dj}-{\beta^*_j})$ is asymptotically zero-mean Gaussian when the elements of $\boldsymbol{\eta}$ are drawn from $\mathcal{N}(0,\sigma^2)$. For specific classes of random matrices, we now show, in Theorem~\ref{th:choice_mu}, that $\frac{\rho}{1+\rho} \leq c_0\sqrt{\frac{\log p}{n}}$ with high probability for some constant $c_0$. This implies that for these random sensing matrices, the choice $\mu := O\left(\sqrt{\frac{\log p}{n}}\right)$ ensures \emph{both} the following: (\textit{i}) asymptotic debiasing for $\boldsymbol{\hat{\beta}_d}$ from \eqref{eq:debiased_beta_W} when $n$ is $\omega((s\log p)^2)$ (see Theorem 8 of \cite{Javanmard2014}), and (\textit{ii})
fulfillment of the sufficient condition $\frac{\rho}{1+\rho} \leq \mu$ for the debiasing matrix $\boldsymbol{W}$ to be computed in closed-form.
If the relation $\frac{\rho}{1+\rho} \leq c_0\sqrt{\frac{\log p}{n}}$ is to be satisfied with high probability, we need an additional (mild) assumption on $\boldsymbol{A}$ as defined below:

\begin{enumerate}
    \item [\textbf{D3:}] $\boldsymbol{\Sigma}$, as defined in \textbf{D1}, is a diagonal matrix, i.e., the elements of the rows of $\boldsymbol{A}$ are uncorrelated.
\end{enumerate}

\begin{theorem}\label{th:choice_mu}
Let $\boldsymbol{A}$ be a $n \times p$ dimensional matrix with independent and identically distributed zero-mean sub-Gaussian rows with uncorrelated entries and sub-Gaussian norm $\kappa:=\|\boldsymbol{\Sigma}^{-1/2}\boldsymbol{a}_{i.}\|_{\psi_2}$, where $n < p$ and $\boldsymbol{\Sigma}:=E[\boldsymbol{a_{i.}a_{i.}}^\top]$. Let $\rho$ be as defined in Theorem~\ref{thm:dual_javanmard_gen}. For any constant $c \in (\sqrt{2}/(1+\sqrt{2}),1)$, if $\boldsymbol{A}$ obeys properties \textbf{D1}, \textbf{D2} and \textbf{D3} and $n \geq \frac{4C_{\max}^2\kappa^4}{C_{\min}^2(1-c)^2} \log p$, then
\begin{equation}\label{eq:mu_bound}
P\left( \frac{\rho}{1+\rho} \leq 2\sqrt{2}\frac{\kappa^2}{c}\frac{C_{\max}}{C_{\min}}\sqrt{\frac{\log p}{n}}\right) \geq 1-\left(\frac{2}{p}+\frac{1}{p^2}\right).
\end{equation}
Furthermore, the choice $\mu:=2\sqrt{2}\frac{\kappa^2}{c}\frac{C_{\max}}{C_{\min}}\sqrt{\frac{\log p}{n}}$ ensures that the optimal debiasing matrix $\boldsymbol{W}$ is given by~\eqref{eq:exact} with high probability.
\end{theorem}
\paragraph{Remarks:}
\begin{enumerate}
\item 
The condition that $c \in (\sqrt{2}/(1+\sqrt{2}),1)$ ensures that when $n \geq \frac{4C_{\max}^2\kappa^4}{C_{\min}^2(1-c)^2} \log p$ and $\mu:=2\sqrt{2}\frac{\kappa^2}{c}\frac{C_{\max}}{C_{\min}}\sqrt{\frac{\log p}{n}}$, then we have $\mu < 1$.
    \item For the choice of $\mu:=2\sqrt{2}\frac{\kappa^2}{c}\frac{C_{\max}}{C_{\min}}\sqrt{\frac{\log p}{n}}$, the optimization problem in \eqref{eq:opt_W_prim} is feasible with high probability under the assumptions \textbf{D1} and \textbf{D2} (as per Theorem 7b of \cite{Javanmard2014}). Additionally, if $\boldsymbol{A}$ satisfies assumption \textbf{D3}, then $\boldsymbol{W}$ has the closed form solution as given in \eqref{eq:exact} with high probability. 
\end{enumerate}
The proof of Theorem \ref{th:choice_mu} is given in Appendix~\ref{subsec:thm2_pf}.
We give a brief overview here.
For an $n\times p$ matrix $\boldsymbol{A}$, let \begin{equation}\label{eq:Ldef}
L := \min_{j\in [p]} \frac{1}{n}\|\boldsymbol{a}_{.,j}\|_2^2
\end{equation} 
and let 
\begin{equation}\label{eq:nudef}
\nu := \max_{i\neq j} \frac{1}{n}|\boldsymbol{a}_{.i}^\top\boldsymbol{a}_{.j}|.
\end{equation} 
Then we have the bound \begin{equation}\label{eq:rho-bound}
\frac{\rho}{1+\rho} \leq \rho = \max_{i\neq j}\frac{|\boldsymbol{a}_{.i}^\top \boldsymbol{a}_{.j}|}{\|\boldsymbol{a}_{.j}\|_2^2} \leq \frac{\nu}{L}.
\end{equation} 
The proof then proceeds by using the following results: Theorem~\ref{th:lower_bnd_L} and Lemma~\ref{le:coherence_bnd}. 
 In Theorem \ref{th:lower_bnd_L}, we show that for an ensemble of sensing matrices satisfying assumptions \textbf{D1}, \textbf{D2} and \textbf{D3}, the parameter $L$ is greater than $c\,C_{\min}$, with high probability, for some constant $c$.

\begin{theorem}\label{th:lower_bnd_L}
    Let $\boldsymbol{A}$ be a $n \times p$ matrix with independently and identically distributed sub-Gaussian rows, where $n < p$. Consider $L$ as defined in \eqref{eq:Ldef}. For any constant $c \in (0,1)$ and $\kappa:=\|\boldsymbol{\Sigma^{-1/2}a_{i.}}\|_{\psi_2}$, if $\boldsymbol{A}$ satisfies properties \textbf{D1} and \textbf{D2} and $n \geq \frac{4C_{\max}^2\kappa^4}{C_{\min}^2(1-c)^2} \log p$, then
    \begin{equation}\label{eq:L_tail}
        P\left(L \geq c\, C_{\min} \right) \geq 1-\frac{2}{p}.
    \end{equation}
\end{theorem}
The proof of Theorem \ref{th:lower_bnd_L} is given in Appendix~\ref{subsec:Thm3_pf}.
In the upcoming Lemma we provide a high probability upper bound on $\nu$ for sensing matrices with independent and identically distributed zero-mean sub-Gaussian rows with uncorrelated entries. 

\begin{lemma}\label{le:coherence_bnd}
    Let $\boldsymbol{A}$ be a $n \times p$ dimensional matrix satisfying assumptions \textbf{D1}, \textbf{D2} and \textbf{D3} and sub-Gaussian norm $\kappa:=\|\boldsymbol{\Sigma^{-1/2}a_{i.}}\|_{\psi_2}$. Define  $\nu$ as in \eqref{eq:nudef}. Then 
\begin{equation}\label{eq:tail_nu}
    P\left(\nu \leq 2\sqrt{2}C_{\max}\kappa^2 \sqrt{\frac{\log p}{n}}\right) \geq 1-\frac{1}{p^2}.
\end{equation} 
\end{lemma}
The proof of Lemma \ref{le:coherence_bnd} is provided in Appendix Sec.\ref{subsec:Le4_pf}. 

The reason that property \textbf{D3} arises in Lemma~\ref{le:coherence_bnd} is as follows. Recall that $\nu$ is the maximum (for $i\neq j$) of the absolute value of the random variables $\frac{1}{n}\boldsymbol{a}_{.i}^{\top}\boldsymbol{a}_{.j}$. These random variables have expectation $\boldsymbol{\Sigma}_{i,j}$, and so it is not possible for $\nu$ to go to zero, asymptotically, unless $\boldsymbol{\Sigma}$ is diagonal.

\section{Empirical Results}

\subsection{Validity of the exact solution}
\paragraph{Aim:} The debiased \textsc{Lasso} can be used to determine the support of the unknown vector $\boldsymbol{\beta^*}$ using hypothesis tests, as per Theorem 8 of \cite{Javanmard2014}. We aim to estimate the support using $p$ hypothesis tests (one per element of $\boldsymbol{\beta^*}$) based on the debiased \textsc{Lasso} estimates using the weights matrix $\boldsymbol{W}$ obtained from the optimization problem in \eqref{eq:opt_W_prim} (denoted by $\boldsymbol{W_o}$), and that obtained from the closed-form expression~\eqref{eq:exact} (denoted by $\boldsymbol{W_e}$), for varying number of measurements $n$. The aim is to also compare these support set estimates with the ground truth support set, and report sensitivity and specificity values (defined below). We will further show the difference in the run-time for both methods.

\paragraph{Signal Generation:} For our simulations, we chose our design matrix $\boldsymbol{A}$ to have elements drawn independently from the standard Gaussian distribution. We synthetically generated signals (i.e., $\boldsymbol{\beta^*}$) with $p=500$ elements in each. The non-zero values of $\boldsymbol{\beta^*}$ were drawn i.i.d.\ from $U(50,1000)$ and placed at randomly chosen indices. We set $s := \|\boldsymbol{\beta^*}\|_0 = 10$ and the noise standard deviation $\sigma:=0.05 \sum_{i=1}^n |\boldsymbol{a_{i.} \beta^*}|/n$. We varied $n \in 
 \{200,250,300,350,400,450,500\}$. We chose $\mu=\rho/(\rho+1)$ where $\rho$ was computed exactly given the sensing matrix $\boldsymbol{A}$.

\paragraph{Sensitivity and Specificity Computation:} Let us denote the debiased \textsc{Lasso} estimates obtained using a matrix $\boldsymbol{\mathsf{W}}$ by $\boldsymbol{\hat{\beta}_{d,\mathsf{W}}}$. We know that asymptotically $\hat{\beta}_{d,\mathsf{W}}(j) \sim \mathcal{N}(\beta^*_j,\sigma^2 \boldsymbol{\mathsf{W}}^{\top}_{.j} \boldsymbol{\mathsf{W}}_{.j})$ for all $j\in [p]$. 
Using this result, $\boldsymbol{\hat{\beta}_{d,\mathsf{W}}}$ was binarized to create a vector $\boldsymbol{\hat{b}_{\mathsf{W}}}$ in the following way: 
For all $j \in [p]$, we set $\hat{b}_{\mathsf{W}}(j) := 1$ if the value of $\hat{\beta}_{\mathsf{W}j}$ was such that the 
the hypothesis $\mathsf{H_{0,j}}: \beta^*_j=0$ was rejected against the alternate $\mathsf{H_{1,j}}: \beta^*_j\ne 0$  at $5 \%$ level of significance. $\hat{b}_{\mathsf{W}}(j)$ was set to 0 otherwise. Note that for the purpose of our simulation, we either have $\boldsymbol{\mathsf{W}} = \boldsymbol{W_o}$ or $\boldsymbol{\mathsf{W}} = \boldsymbol{W_e}$. The binary vectors corresponding to these choices of $\boldsymbol{\mathsf{W}}$ are respectively denoted by $\boldsymbol{\hat{b}_{W_o}}$ and $\boldsymbol{\hat{b}_{W_e}}$.

A ground truth binary vector $\boldsymbol{b^*}$ was created such that $b^*_j := 1$ at all locations $j$ where $\beta^*_j \ne 0$ and $b^*_j := 0$ otherwise. Sensitivity and specificity values were computed by comparing corresponding entries of $\boldsymbol{b^*}$ to those in $\boldsymbol{\hat{b}_{W_o}}$ and $\boldsymbol{\hat{b}_{W_e}}$. 
Considering the matrix $\boldsymbol{\mathsf{W}}$, we declared an element to be a \textit{true defective} if $b^*_j = 1$ and $\hat{b}_{\mathsf{W},j} = 1$, and a \textit{false defective} if $b^*_j = 0$ but $\hat{b}_{\mathsf{W},j} \ne 0$. We declare it to be a \textit{false non-defective} if $b^*_j = 0$ but $\hat{b}_{\mathsf{W},j} \ne 0$, and a \textit{true non-defective} if $\beta^*_j = 0$ and $\hat{b}_{\mathsf{W},j} = 0$. The \textbf{sensitivity} for $\boldsymbol{\beta^*}$ is defined as ($\#$ true defectives)/($\#$ true defectives + $\#$ false non-defectives) and
\textbf{specificity} for $\boldsymbol{\beta^*}$ is defined as ($\#$ true non-defectives)/($\#$ true non-defectives + $\#$ false defectives).

\paragraph{Results:} For obtaining $\boldsymbol{W_o}$, the optimization routine was executed using the \texttt{lsqlin} package in MATLAB. The sensivitiy and specificity were averaged over 25 runs with independent noise instances.
\begin{table}
    \centering
    \begin{tabular}{cccccccc}
    \toprule
    & \multicolumn{2}{c}{sensitivity} & \multicolumn{2}{c}{specificity} & 
    \multicolumn{2}{c}{time (in s)} & \\
    \cmidrule(lr){2-3}\cmidrule(lr){4-5}\cmidrule(lr){6-7}
       $n$  & $\boldsymbol{W_o}$ & $\boldsymbol{W_e}$ & $\boldsymbol{W_o}$ & $\boldsymbol{W_e}$ & $\boldsymbol{W_o}$ & $\boldsymbol{W_e}$& $\frac{\|\boldsymbol{W}_o-\boldsymbol{W}_e\|_F}{\|\boldsymbol{W}_e\|_F}$ \\\midrule
        200 & 0.6742 & 0.6742 & 0.8592 & 0.8592 & $3.88 \times 10^2$ & $1.11 \times 10^{-3}$ & $6.68 \times 10^{-10}$\\
        250 & 0.7229 & 0.7229 & 0.9063 & 0.9063 & $5.22 \times 10^2$ & $1.72 \times 10^{-3}$ & $2.31 \times 10^{-8}$\\
        300 &  0.8071 & 0.8071 & 0.9427 & 0.9427 & $3.29 \times 10^2$ & $2.25 \times 10^{-3}$ & $2.73 \times 10^{-7}$ \\
        350 & 0.8554 & 0.8554 & 0.9719 & 0.9719& $4.77 \times 10^2$ & $3.88 \times 10^{-3}$ & $2.56 \times 10^{-7}$\\
        400 & 0.9275 & 0.9275 & 0.9855 & 0.9855 & $5.59 \times 10^2$ & $7.82 \times 10^{-3}$ & $4.76 \times 10^{-7}$\\
        450 & 0.9781 & 0.9781 & 0.9909 & 0.9909 & $7.15 \times 10^2$ & $4.27 \times 10^{-2}$ & $5.29 \times 10^{-7}$\\
        500 & 0.9985 & 0.9985 & 0.9992 & 0.9992 & $8.03 \times 10^2$ & $7.56 \times 10^{-2}$ & $8.22 \times 10^{-7}$\\
    \bottomrule
    \end{tabular}
        \caption{Sensitivity and Specificity of hypothesis test using debiased estimates obtain from $\boldsymbol{W_o}$ (optimization method) and $\boldsymbol{W_e}$ (closed-form expression from~\eqref{eq:exact}) with its corresponding runtime in seconds for varying number of measurements. The fixed parameters are $p=500, s=10, \sigma:=0.05 \sum_{i=1}^n |\boldsymbol{a_{i.} \beta^*}|/n$. We set $\mu=\rho/(\rho+1)$ where $\rho$ is computed exactly for the chosen  sensing matrix $\boldsymbol{A}$.}
    \label{tab:sens_spec_Gauss}
\end{table}
    
In Table~\ref{tab:sens_spec_Gauss}, we can see that the sensitivity as well as the specificity of the hypothesis tests for $\boldsymbol{W_o}$ and $\boldsymbol{W_e}$ are equal. We further report the relative difference between $\boldsymbol{W_o}$ and $\boldsymbol{W_e}$ in the Frobenius norm. We can clearly see that the difference is negligible, which is consistent with Theorem~\ref{thm:dual_javanmard_gen}. Furthermore, we see that using the closed-form expression in~\eqref{eq:exact} saves significantly on time (by a factor of at least $10^4$). 
While the computational efficiency of the iterative approach can be improved by developing a specialized solver for problems of the form~\eqref{eq:opt_W_prim}, no iterative method is expected to outperform directly computing the simple closed-form expression~\eqref{eq:exact}.

\subsection{Difference between $\boldsymbol{W}_e$ and $\boldsymbol{W}_o$ for varying choices of $\mu$} \label{subsec:mu_line}

\paragraph{Aim:} In Theorem \ref{thm:dual_javanmard_gen}, we show that if $\frac{\rho}{1+\rho} \leq \mu < 1$, then the closed form solution of \eqref{eq:exact} represented by $\boldsymbol{W}_e$ is the same as the solution of the optimization problem  given in \eqref{eq:opt_W_prim} represented by $\boldsymbol{W}_o$. In this subsection, we investigate the difference between $\boldsymbol{W}_o$ and $\boldsymbol{W_e}$ for $\mu <  \frac{\rho}{1+\rho}$ as well as in the range $ \frac{\rho}{1+\rho} \leq \mu < 1$. We report the difference between $\boldsymbol{W}_e$ and $\boldsymbol{W}_o$ in terms of the \textit{Relative Error} given by $\left(\frac{\|\boldsymbol{W}_o-\boldsymbol{W}_e\|_F}{\|\boldsymbol{W_e}\|_F}\right)$ for $\mu=0.2,0.21,0.22,\ldots,0.60$.

\paragraph{Sensing matrix properties:} For this experiment, we fixed $n=80,p=100$. We ran this experiment for two different $n\times p$ sensing matrices $\boldsymbol{A}$ with elements drawn from: (1) i.i.d.\ Gaussian and, (2) i.i.d.\ Rademacher. In Figure~\ref{fig:mu_frob}, we plot $\mu$ vs $\left(\frac{\|\boldsymbol{W}_o-\boldsymbol{W}_e\|_F}{\|\boldsymbol{W_e}\|_F}\right)$ for both of these matrices on a log scale. The exact value of $ \frac{\rho}{1+\rho}$ is given by a black vertical line in each case.

\paragraph{Observation:} We see that for both the plots in Figure \ref{fig:mu_frob}, the relative error decreases with increase in $\mu$ for $\mu<  \frac{\rho}{1+\rho}$. For $\mu \geq  \frac{\rho}{1+\rho}$, the relative error is very small with fluctuations primarily due to the solver tolerances in \texttt{lsqlin} when computing $\boldsymbol{W}_o$.
\begin{figure}
    \centering
    \includegraphics[height=1.50in]{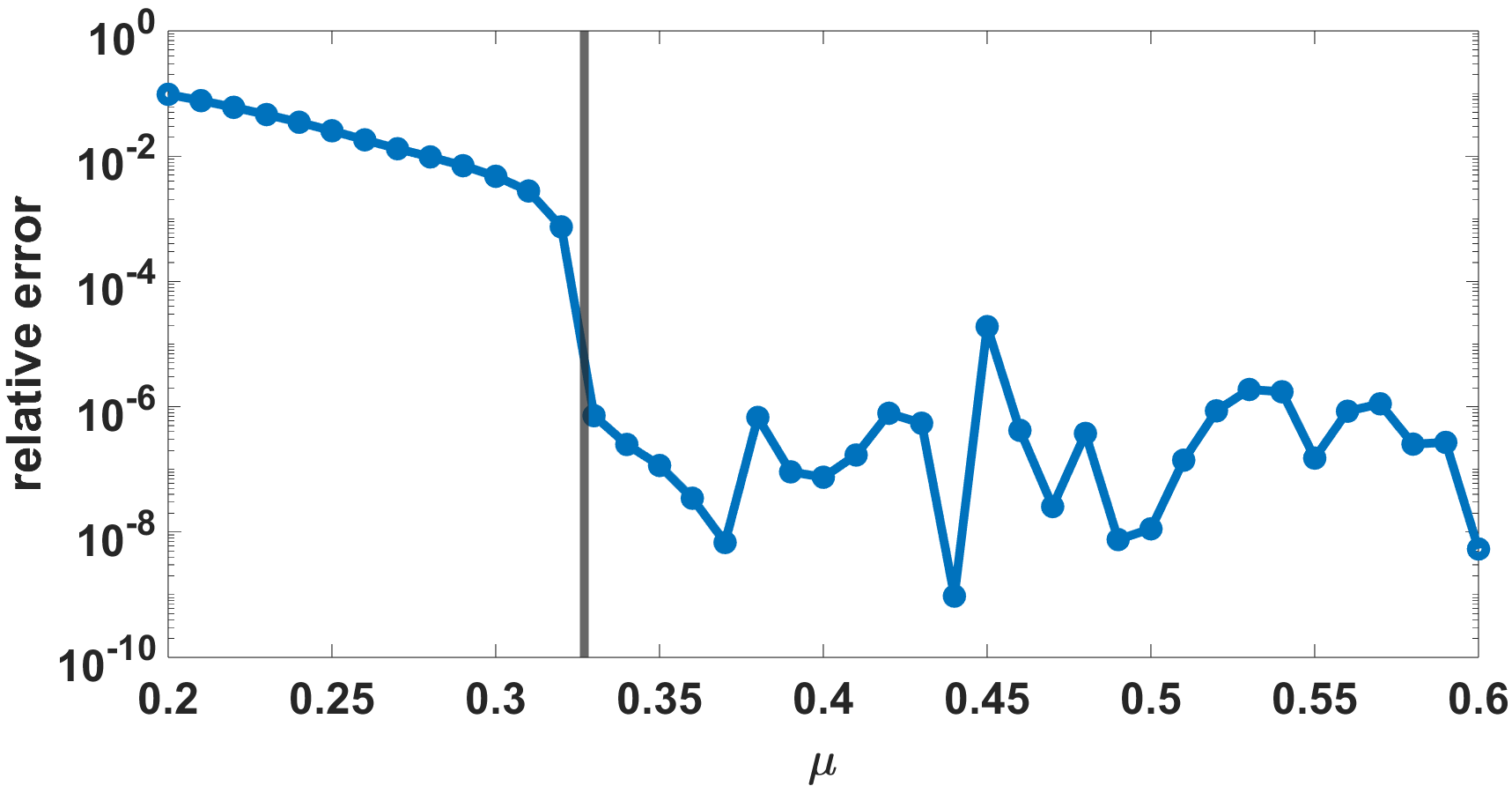}
    \includegraphics[height=1.50in]{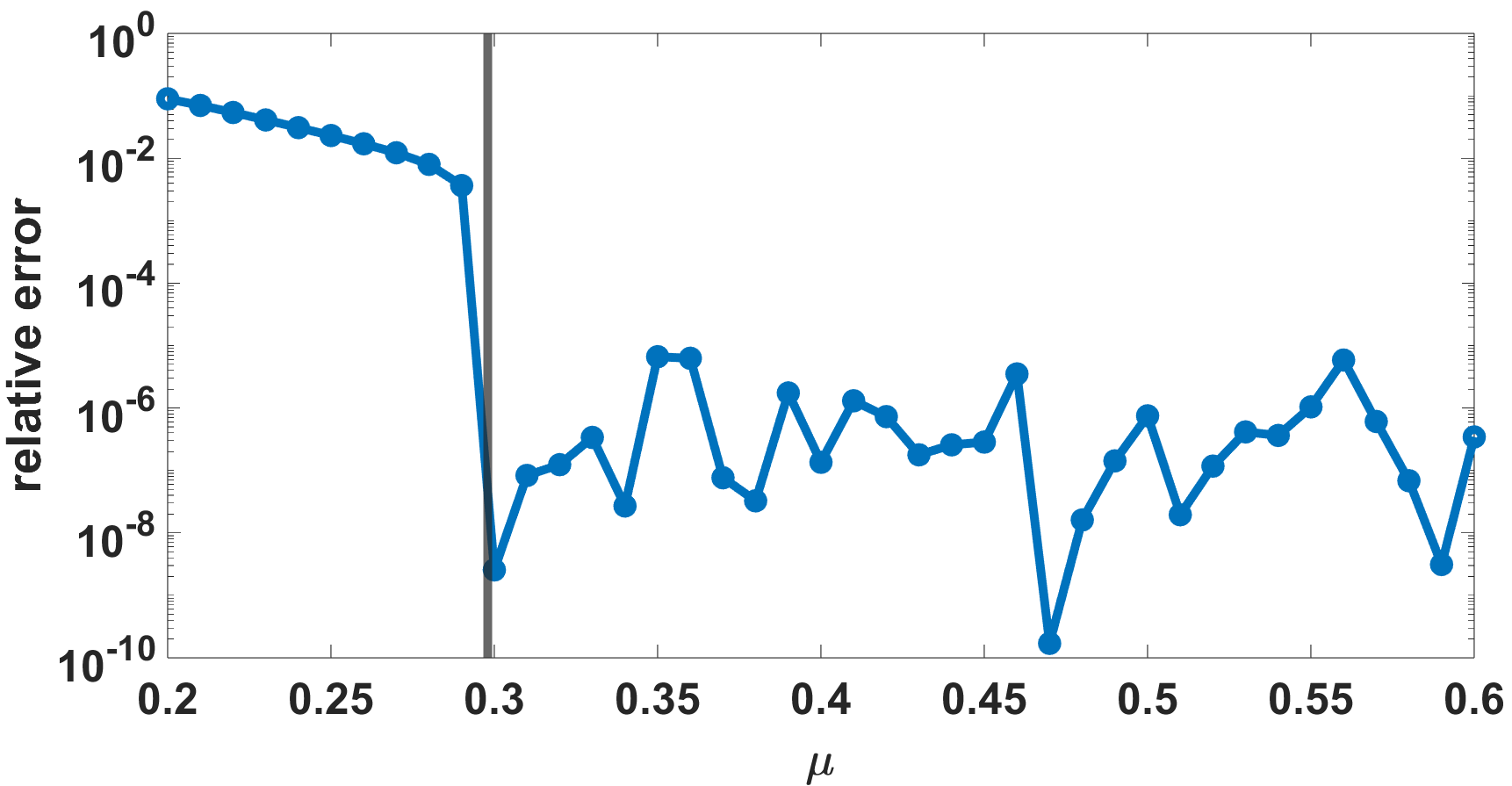}
    \caption{Line plot of $\mu$ vs relative error $ \left(\frac{\|\boldsymbol{W}_o-\boldsymbol{W}_e\|_F}{\|\boldsymbol{W_e}\|_F}\right)$ (in log scale) for two $80 \times 100$ dimensional sensing matrices: (left) i.i.d.\ Gaussian and (right) i.i.d.\ Rademacher. The exact value of $ \frac{\rho}{1+\rho}$ is given by the black vertical line. The value of $ \frac{\rho}{1+\rho}$ is $0.327$ for the Gaussian sensing matrix (left) and $0.298$ for the Rademacher sensing matrix  (right). Here, $\boldsymbol{W_o}$ is the solution of the optimization problem in \eqref{eq:opt_W_prim} and $\boldsymbol{W}_e$ is computed as in \eqref{eq:exact}.}
    \label{fig:mu_frob}
\end{figure}
Furthermore the decrease in relative error is sharp after the value of $\mu$ crosses $ \frac{\rho}{1+\rho}$.

\section{Conclusion}
In this article, we reformulate the optimization problem to obtain $\boldsymbol{M}$ (the approximate inverse of the covariance matrix of the rows of the sensing matrix $\boldsymbol{A}$) in \cite{Javanmard2014} and further provide an exact, closed-form optimal solution to the reformulated problem under assumptions on the pairwise inner products of the columns of $\boldsymbol{A}$. For sensing matrices with i.i.d.\ zero-mean sub-Gaussian rows that have diagonal covariance, the debiased \textsc{Lasso} estimator, based on this closed-form solution, has entries that are asymptotically zero-mean and Gaussian. The exact solution significantly improves the time efficiency for debiasing the \textsc{Lasso} estimator, as shown in the numerical results. Our method is particularly useful for debiasing in streaming settings where new measurements arrive on the fly. 

\appendix
\section{Additional proofs}
\subsection{Proof of Theorem \ref{thm:dual_javanmard_gen}}
\label{sec:proof_thm_dual}
\paragraph{Primal feasibility:}
If $\frac{\rho}{1+\rho}\leq \mu \leq 1$ then we have that $\mu + \mu\rho \geq \rho$ which implies that $0\leq (1-\mu)\rho \leq \mu$. The choice of $\boldsymbol{w}_{.j}$ given by \eqref{eq:exact} is primal feasible since 
\begin{equation}
\left\|\frac{1}{n}\boldsymbol{A}^\top\frac{(1-\mu)}{\frac{\|\boldsymbol{a_{.j}}\|_2^2}{n}}\boldsymbol{a}_{.j} - \boldsymbol{e}_j\right\|_{\infty} \leq \max\{ \mu, |(1-\mu)\rho|\} = \mu.
\end{equation}
To see why this is true, note that for index $j$, the LHS is upper bounded by $\mu$, otherwise it is upper bounded by $|(1-\mu)\rho|$. 


\paragraph{Primal objective function value:} The primal objective function value is given by $\frac{1}{n}\|\boldsymbol{w}_{.j}\|^2_2 =\dfrac{(1-\mu)^2}{(\|\boldsymbol{a}_{.j}\|_2^2/n)^2}\|\boldsymbol{a}_{.j}\|_2^2/n =\dfrac{(1-\mu)^2}{\|\boldsymbol{a_{.j}}\|_2^2/n}$. 

\paragraph{The Fenchel dual problem:} Consider an optimization problem of the form for a fixed $j \in [p]$:
\begin{equation} 
\label{eq:abstract-primal}\inf_{\boldsymbol{w}} f(\boldsymbol{w}) + g_j\left(\frac{1}{n}\boldsymbol{A}^\top \boldsymbol{w}\right)\end{equation}
where $f$ and $g_j$ are extended real-valued convex functions. The Fenchel dual (see Chapter 3 of \cite{BorweinLewis2006}) is
\begin{equation}\label{eq:abstract-dual} \sup_{\boldsymbol{u}} -f^*\left(\frac{1}{n}\boldsymbol{A}\boldsymbol{u}\right) - g_j^*(-\boldsymbol{u})\end{equation}
where $f^*$ and $g_j^*$ are the convex conjugates of $f$ and $g_j$ respectively. The Fenchel dual satisfies weak duality (see Chapter 3 of \cite{BorweinLewis2006}), i.e., for any $\boldsymbol{w}$ and $\boldsymbol{u}$, 
\[ f(\boldsymbol{w}) + g_j\left(\frac{1}{n}\boldsymbol{A}^\top\boldsymbol{w}\right) \geq -f^*\left(\frac{1}{n}\boldsymbol{A}\boldsymbol{u}\right) - g_j^*(-\boldsymbol{u}).\]
In our setting, for a fixed $j$, we consider 
\begin{equation}
\label{eq:f_g}
f(\boldsymbol{w}) := \frac{1}{n}\|\boldsymbol{w}\|^2\quad\textup{and}\quad g_j(\boldsymbol{w}) := \begin{cases} 0 & \textup{if $\|\boldsymbol{w} - \boldsymbol{e}_j\|_{\infty} \leq \mu$}\\ \infty & \textup{otherwise}\end{cases}.
\end{equation}
Then, for the same $j$, we have their convex conjugates from Lemma~\ref{le:convex conjugates}: 
\begin{eqnarray}
f^*(\boldsymbol{u})& =& \sup_{\boldsymbol{w}}  \boldsymbol{u}^\top\boldsymbol{w} - f(\boldsymbol{w}) = \frac{n}{4}\|\boldsymbol{u}\|^2, \\
g_j^*(\boldsymbol{u}) &=& \sup_{\boldsymbol{w}} \boldsymbol{u}^\top\boldsymbol{w} - g_j(\boldsymbol{w}) = \sup_{\|\boldsymbol{w}-\boldsymbol{e}_j\|_{\infty}\leq \mu}\boldsymbol{u}^\top\boldsymbol{w} = u_j + \mu\|\boldsymbol{u}\|_1.
\end{eqnarray}
This gives a dual problem in the form
$\sup_{\boldsymbol{u}} - \frac{1}{4n}\boldsymbol{u}^\top \boldsymbol{A}^\top \boldsymbol{A}\boldsymbol{u} + u_j - \mu\|\boldsymbol{u}\|_1$.

\noindent The point $\boldsymbol{u} := \dfrac{2(1-\mu)\boldsymbol{e}_j}{\|\boldsymbol{a_{.j}}\|_2^2/n}$ is feasible for the dual (trivially, as there are no constraints). 

\paragraph{Dual objective function value:} Plugging in $\boldsymbol{u} = \dfrac{2(1-\mu)\boldsymbol{e}_j}{\|\boldsymbol{a_{.j}}\|_2^2/n}$, the corresponding dual objective function value is 
\begin{align*}
-\frac{1}{4n}\boldsymbol{u}^\top \boldsymbol{A}^\top \boldsymbol{A}\boldsymbol{u} + u_j - \mu\|\boldsymbol{u}\|_1 & = -\frac{1}{4n} \|\boldsymbol{a}_{.j}\|^2 \frac{4(1-\mu)^2}{(\|\boldsymbol{a_{.j}}\|_2^2/n)^2} +\frac{2(1-\mu)}{\|\boldsymbol{a_{.j}}\|_2^2/n} - \mu\frac{2(1-\mu)}{\|\boldsymbol{a_{.j}}\|_2^2/n}\\
&= -\frac{(1-\mu)^2}{\|\boldsymbol{a_{.j}}\|_2^2/n}+2\frac{(1-\mu)^2}{\|\boldsymbol{a_{.j}}\|_2^2/n} = \frac{(1-\mu)^2}{\|\boldsymbol{a_{.j}}\|_2^2/n}.\end{align*}

Since the primal solution and the dual objective function values are equal, it follows that an optimal solution for the primal is $\dfrac{(1-\mu)}{\|\boldsymbol{a_{.j}}\|_2^2/n}\boldsymbol{a}_{.j}$, and that an optimal solution to the dual is $\dfrac{2(1-\mu)}{\|\boldsymbol{a_{.j}}\|_2^2/n}\boldsymbol{e}_j$.

We have shown that if $\rho/(1+\rho)\leq \mu \leq 1$ then the optimal solution of~\eqref{eq:opt_W_prim} is given by~\eqref{eq:exact}. Now consider the case when $\mu < \rho/(1+\rho)$. This implies $\mu < (1-\mu)\rho$. Let $i,j\in [p]$ (with $i\neq j$) be such that $\rho = |\boldsymbol{a}_{.i}^\top \boldsymbol{a}_{.j}|/\|\boldsymbol{a}_{.j}\|_2^2$. Then plugging in the expression $\boldsymbol{w}_{.j} := \frac{n(1-\mu)}{\|\boldsymbol{a_{.j}}\|_2^2}\boldsymbol{a}_{.j}$ from~\eqref{eq:exact} into the constraint of~\eqref{eq:opt_W_prim} we have, \[\left\|\frac{1}{n}\boldsymbol{A}^\top\frac{(1-\mu)}{\frac{\|\boldsymbol{a_{.j}}\|_2^2}{n}}\boldsymbol{a}_{.j} - \boldsymbol{e}_j\right\|_{\infty} \geq (1-\mu)|\boldsymbol{a}_{.i}^\top \boldsymbol{a}_{.j}|/\|\boldsymbol{a}_{.j}\|_2^2 = (1-\mu)\rho > \mu.\] This shows that $\boldsymbol{w}_{.j}$ (defined in~\eqref{eq:exact}) is not feasible for~\eqref{eq:opt_W_prim} when $\mu < \rho/(1+\rho)$, and so is certainly not optimal.

Finally, consider the case when $\mu>1$. If $\mu \geq 1$, then the unique optimal solution of ~\eqref{eq:opt_W_prim} is $\boldsymbol{w}_{.j} = \boldsymbol{0}$. This is because $\boldsymbol{0}$ is feasible and is the global minimizer of the objective function.
However, when $\mu>1$, the formula~\eqref{eq:exact} does not give the value $\boldsymbol{0}$, and so is not the optimal solution to~\eqref{eq:opt_W_prim}.

This concludes the proof that $\frac{\rho}{1+\rho} \leq \mu \leq 1$ is necessary and sufficient condition for the expression given in \eqref{eq:exact} to be optimal.

\subsection{Proof of Theorem \ref{th:choice_mu}}
\label{subsec:thm2_pf}
In Theorem \ref{th:lower_bnd_L} we show that if $n \geq \frac{4C_{\max}^2\kappa^4}{C_{\min}^2(1-c)^2} \log p$ then  $ \frac{\rho}{1+\rho} \leq \frac{\nu}{L} \leq \frac{\nu}{c C_{\min}}$ with probability at least $1-\frac{2}{p}$. Now, Lemma \ref{le:coherence_bnd} shows that $\nu \leq 2\sqrt{2}C_{\max}{\kappa^2} \sqrt{\frac{\log p}{n}}$ with probability at least $1-\frac{1}{p^2}$. Hence, by union bound, we have,
\begin{equation}\label{eq:mu_bound_penul}
     P\left( \frac{\rho}{1+\rho} \leq \frac{\nu}{c\, C_{\min}} \leq  2\sqrt{2}\frac{\kappa^2}{c}\frac{C_{\max}}{C_{\min}}\sqrt{\frac{\log p}{n}}\right) \geq 1-\left(\frac{2}{p}+\frac{1}{p^2}\right)
\end{equation}
 This completes the proof. 

\subsection{Proof of Theorem \ref{th:lower_bnd_L}}\label{subsec:Thm3_pf}
 We have for all $j \in [p]$, $\frac{\|\boldsymbol{a}_{.j}\|_2^2}{n}=\frac{1}{n}\sum_{i=1}^n a_{ij}^2$. Since the $\boldsymbol{a}_{i.}$ (for $i \in [n]$) are sub-Gaussian, the $a_{ij}$ are sub-Gaussian for each $j\in [p]$ and $\|a_{ij}\|_{\psi_2} \leq \|\boldsymbol{a}_{i.}\|_{\psi_2}$.
 By the definition of the sub-Gaussian norm (see footnote in Sec.~\ref{subsec:debias_LASSO} with $q = 2$), we know that  \begin{equation}\label{eq:subg1} \frac{1}{2}E[a_{ij}^2] \leq \|a_{ij}\|_{\psi_2}^2 = \|\boldsymbol{e}_j^\top \boldsymbol{a}_{i.}\|_{\psi_2}^2 \leq \|\boldsymbol{a}_{i.}\|_{\psi_2}^2.
 \end{equation}
 Recall that $\kappa:=\|\boldsymbol{\Sigma}^{-1/2}\boldsymbol{a}_{i.}\|_{\psi_2}$ in property \textbf{D1} of sensing matrix $\boldsymbol{A}$. We have
 \begin{align}
     \|\boldsymbol{a}_{i.}\|_{\psi_2} & = \sup_{\boldsymbol{v}\in S^{p-1}}\left\|(\boldsymbol{\Sigma}^{1/2}\boldsymbol{v})^\top \boldsymbol{\Sigma}^{-1/2}\boldsymbol{a}_{i.}\right\|_{\psi_2}\nonumber\\
     & = \sup_{\boldsymbol{v}\in S^{p-1}}\|\boldsymbol{\Sigma}^{1/2}\boldsymbol{v}\|_2 \left\|\frac{1}{\|\boldsymbol{\Sigma}^{1/2}\boldsymbol{v}\|_2}(\boldsymbol{\Sigma}^{1/2}\boldsymbol{v})^\top \boldsymbol{\Sigma}^{-1/2}\boldsymbol{a}_{i.}\right\|_{\psi_2}\nonumber\\
     & \leq \sup_{\boldsymbol{v}\in S^{p-1}}\|\boldsymbol{\Sigma}^{1/2}\boldsymbol{v}\|_2 \sup_{\boldsymbol{z}\in S^{p-1}}\left\|\frac{1}{\|\boldsymbol{\Sigma}^{1/2}\boldsymbol{z}\|_2}(\boldsymbol{\Sigma}^{1/2}\boldsymbol{z})^\top \boldsymbol{\Sigma}^{-1/2}\boldsymbol{a}_{i.}\right\|_{\psi_2}\nonumber\\
     & \leq \sigma_{\max}(\boldsymbol{\Sigma}^{1/2}) \|\boldsymbol{\Sigma}^{-1/2}\boldsymbol{a}_{i.}\|_{\psi_2}\nonumber\\
     & \leq \sqrt{C_{\max}}\,\kappa,\label{eq:subg2}
     \end{align} where $C_{\max}$ is defined in property \textbf{D2}. Therefore, we obtain $E[a_{ij}^2]\leq 2 \|\boldsymbol{a}_{i.}\|_{\psi_2}^2 \leq 2C_{\max}\kappa^2$.
 From the definition of eigenvalues, for any $\boldsymbol{x} \in \mathbb{R}^p$, $\boldsymbol{x}^\top \boldsymbol{\Sigma x} \geq \sigma_{\min}(\boldsymbol{\Sigma}) \|\boldsymbol{x}\|_2^2 \geq C_{\min}\|\boldsymbol{x}\|_2^2$. Putting $\boldsymbol{x}=\boldsymbol{e}_j$, where $\boldsymbol{e}_j$ is the $j^{\text{th}}$ column of $\boldsymbol{I_p}$, we have, $\Sigma_{jj}\geq C_{\min}$. Since $E[a_{ij}^2]=\Sigma_{jj} \geq C_{\min}$,
we have, $E\left[\frac{1}{n}\sum_{i=1}^n a_{ij}^2\right] \geq C_{\min}$. 

For a given $j \in [p]$, the variables $a_{ij}^2$ are independent for all $i \in [n]$. Hence, using the concentration inequality of Theorem 3.1.1 and Equation (3.3) of \cite{Vershynin2018}, we have for $t>0$\footnote{We have set $c = 1/2$, $\delta := t$ and $K := 2\sqrt{C_{\max}}\kappa$ in Equation (3.3) and the equation immediately preceding it in  \cite{Vershynin2018}},
 \begin{equation}\label{eq:conc_ineq_subG}
     P\left(\left|\|\boldsymbol{a}_{.j}\|_2^2/n- E[\|\boldsymbol{a}_{.j}\|_2^2/n]\right| \geq t \right) \leq 2e^{-\frac{nt^2}{2C_{\max}^2\kappa^4}}.
 \end{equation}
Using the left-sided inequality of \eqref{eq:conc_ineq_subG}, we have,
\begin{equation}\label{eq:left_conc_ineq_subG}
    P\left(\|\boldsymbol{a}_{.j}\|_2^2/n \leq E[\|\boldsymbol{a}_{.j}\|_2^2/n] - t\right)\leq 2e^{-\frac{nt^2}{2C_{\max}^2\kappa^4}}.
\end{equation}
Using $E[\|\boldsymbol{a}_{.j}\|_2^2/n] \geq {C_{\min}}$, \eqref{eq:left_conc_ineq_subG} can be rewritten as follows for $t>0$:
\begin{equation}\label{eq:conc_lower_subG}
    P\left(\|\boldsymbol{a}_{.j}\|_2^2/n \leq C_{\min}- t\right)\leq 2e^{-\frac{nt^2}{2C_{\max}^2\kappa^4}}.
\end{equation}
 Using the union bound on \eqref{eq:conc_lower_subG} over $j \in [p]$, we obtain the following lower tail bound on $L$:
\begin{equation}\label{eq:lower_tail_L}
    P(L \leq C_{\min}- t) \leq 2pe^{-\frac{nt^2}{2C_{\max}^2\kappa^4}}.
\end{equation}
Putting $t:=2C_{\max}\kappa^2\sqrt{\frac{\log p}{n}}$ in \eqref{eq:lower_tail_L}, we obtain:
\begin{equation}\label{eq:lower_bound_L}
    P\left(L \leq C_{\min}\left(1-2\frac{C_{\max}}{C_{\min}}\kappa^2\sqrt{\frac{\log p}{n}}\right) \right) \leq \frac{2}{p}.
\end{equation}
For some constant $c\in (0,1)$, if $n \geq \frac{4C_{\max}^2\kappa^4}{C_{\min}^2(1-c)^2} \log p$, then \eqref{eq:lower_bound_L} becomes:
\begin{equation}\label{eq:final_tail_L}
    P(L\leq c\, C_{\min}) \leq \frac{2}{p} \implies P(L\geq c\, C_{\min}) \geq 1-\frac{2}{p}.
\end{equation}
 This completes the proof. 

\subsection{Proof of Lemma~\ref{le:coherence_bnd}}\label{subsec:Le4_pf}
We have $ \frac{1}{n}| \boldsymbol{a}_{.l}^\top\boldsymbol{a}_{.j}|=\frac{1}{n}\sum_{i=1}^n a_{ij} a_{il}$. Here, for given $j \ne l$, we know that $a_{ij}$ and $a_{il}$ are independent zero-mean sub-Gaussian random variables. From from~\eqref{eq:subg1} and~\eqref{eq:subg2} we know that their sub-Gaussian norm is at most $\sqrt{C_{\max}}\kappa$ for all $i \in [n]$. Using Lemma 2.7.7 of \cite{vershynin2018high}, we have that for all $i \in [n]$, $a_{ij}a_{il}$ are independent sub-Exponential random variables with sub-exponential norm at most $C_{\max}\kappa^2$. Moreover, $E[a_{ij}a_{il}] = \boldsymbol{\Sigma}_{jl} = 0$ by property \textbf{D3}. Hence, using Bernstien's inequality for averages of independent zero-mean, sub-exponential random variables, given in Corollary 2.8.3 of \cite{vershynin2018high}, we have for any $t>0$,
\begin{equation}\label{eq:bern_sub_exp}
    P\left(\frac{1}{n}\sum_{i=1}^n a_{ij} a_{il} \geq t\right) \leq 2 e^{-\frac{nt^2}{2C_{\max}^2\kappa^4}}
\end{equation}
Hence, using the symmetry of the inner product and a union bound, we have
\begin{equation}\label{eq:union_sub_exp}
    P\left( \max_{l\neq j} \frac{1}{n}| \boldsymbol{a}_{.l}^\top\boldsymbol{a}_{.j}| \geq t\right) =P\left( \max_{l< j} \frac{1}{n}| \boldsymbol{a}_{.l}^\top\boldsymbol{a}_{.j}| \geq t\right) \leq 2\binom{p}{2} e^{-\frac{nt^2}{2C_{\max}^2\kappa^4}}
\end{equation}
Taking $t=2\sqrt{2}C_{\max}\kappa^2 \sqrt{\frac{\log p}{n}}$, we have,
\begin{equation}\label{eq:tail_nu}
    P\left(\nu \geq 2\sqrt{2}C_{\max}\kappa^2 \sqrt{\frac{\log p}{n}}\right) \leq \frac{1}{p^2}.
\end{equation}
This completes the proof.

\subsection{Convex conjugates}
The convex conjugate of a function $f(\boldsymbol{w})$ is defined as:
    \begin{equation}
    f^*(\boldsymbol{u}) = \sup_{\boldsymbol{w}} \big( \boldsymbol{u}^\top\boldsymbol{w} - f(\boldsymbol{w}) \big).
    \end{equation}
    The following result gives the convex conjugates of the functions needed in the proof of Theorem~\ref{thm:dual_javanmard_gen}.
\begin{lemma}\label{le:convex conjugates}

\begin{enumerate}
\item If $f(\boldsymbol{w}) = \frac{1}{n} \|\boldsymbol{w}\|_2^2$, then its convex conjugate is
$f^*(\boldsymbol{u}) = \frac{n}{4} \|\boldsymbol{u}\|_2^2$.

\item If $g_j$  is the indicator function of the convex set $\{ \boldsymbol{w} \in \mathbb{R}^p \mid \|\boldsymbol{w} - \boldsymbol{e_j}\|_\infty \leq \mu \}$, i.e., 
\[ g_j(\boldsymbol{w}) = \begin{cases} 0 & \textup{if $\|\boldsymbol{w} - \boldsymbol{e_j}\|_\infty \leq \mu$}\\\infty & \textup{otherwise,}\end{cases}
\]
then its convex conjugate is 
$g_j^*(\boldsymbol{u}) =u_j + \mu \|\boldsymbol{u}\|_1$.
\end{enumerate}
\end{lemma}
\begin{proof}
\begin{enumerate}
    \item We can write $f(\boldsymbol{w}) = \frac{1}{2} \boldsymbol{w}^\top \boldsymbol{Q}\boldsymbol{w}$ where $\boldsymbol{Q} := \frac{2}{n}\boldsymbol{I_p}$ is positive definite (and has size $p \times p$). From Example 3.2.2 of \cite{boyd2004convex}, the 
    convex conjugate of a positive definite quadratic form is 
    \[ f^*(\boldsymbol{u}) = \frac{1}{2}\boldsymbol{u}^\top \boldsymbol{Q}^{-1}\boldsymbol{u} = \frac{1}{2}\boldsymbol{u}^\top\left(\frac{2}{n}\boldsymbol{I_p}\right)^{-1}\boldsymbol{u} = \frac{n}{4}\|\boldsymbol{u}\|_2^2.\]
    \item If $g_j$ is the indicator function of the set $C$, 
    the convex conjugate is given by
    \begin{equation}
    g_j^*(\boldsymbol{u}) = \sup_{\boldsymbol{w} \in C} \boldsymbol{u}^\top \boldsymbol{w},
    \end{equation}
    where $C = \{ \boldsymbol{w} \in \mathbb{R}^p \mid \|\boldsymbol{w} - \boldsymbol{e_j}\|_\infty \leq \mu \}$. This implies that $w_i \in [e_{ji}- \mu, e_{ji} + \mu], \forall i$. (Note that $e_{ij}=1$ if $i=j$ and $0$ otherwise.)
    To maximize $\boldsymbol{u}^\top\boldsymbol{w} = \sum_{i=1}^p u_i w_i$, the optimal $w_i$ can be chosen as
    \begin{equation}
    w_i = 
    \begin{cases}
        e_{ji} + \mu & \text{if } u_i \geq 0, \\
        e_{ji} - \mu & \text{if } u_i < 0.
    \end{cases}
    \end{equation}
    Substituting into $\boldsymbol{u}^\top\boldsymbol{w}$, we obtain $\boldsymbol{u}^\top\boldsymbol{w} = \sum_{i=1}^p u_i \big( e_{ji} + \mu \, \text{sign}(u_i) \big)$,
    where $\text{sign}(u_i)$ is the sign of $u_i$. Simplifying, we have $\boldsymbol{u}^\top\boldsymbol{w} = u_j + \mu \sum_{i=1}^p |u_i|$.
    Thus, we have 
\begin{equation}    
g_j^*(\boldsymbol{u}) =u_j + \mu \|\boldsymbol{u}\|_1.
\end{equation}
\end{enumerate}
\end{proof}
\bibliographystyle{plain}
\bibliography{fastdebiasing_arxiv}
\end{document}